\documentclass{article} % right eq number
\usepackage[utf8]{inputenc}

\usepackage{txie}
\usepackage[final]{neurips_wrl2022}
\usepackage{bbm}

\setlist{leftmargin=*,topsep=0pt}

\def\shownotes{1}  % 1 -- show author notes; 0 -- not show
\ifnum\shownotes=1
\newcommand{\authnote}[2]{\noindent$^{\text{\fontfamily{cmtt}\em #1:}}\langle${\sf\small #2}$\rangle$}
\else
\newcommand{\authnote}[2]{}
\fi
\def\algo{ARMOR\xspace}
\def\algofull{Adversarial Models for Offline Reinforcement Learning\xspace}
\def\algohighlighted{\textbf{A}dve\textbf{r}sarial \textbf{M}odels for \textbf{O}ffline \textbf{R}inforcement Learning\xspace}

\usepackage[T1]{fontenc}
\usepackage[scaled=.91]{helvet}

\usepackage{inconsolata}

\allowdisplaybreaks

\title{ARMOR: A Model-based Framework for Improving Arbitrary Baseline Policies with Offline Data}
\usepackage{authblk}
\author[*]{\bf Tengyang Xie}
\author[$\dagger$]{\bf Mohak Bhardwaj}
\author[*]{\bf Nan Jiang}
\author[$\ddagger$]{\bf Ching-An Cheng}
\affil[*]{University of Illinois at Urbana-Champaign}
\affil[$\dagger$]{University of Washington}
\affil[$\ddagger$]{Microsoft Research}
\affil[ ]{\normalsize \texttt{\{tx10, nanjiang\}@illinois.edu}\quad \texttt{mohakb@cs.washington.edu}\quad \texttt{chinganc@microsoft.com}}

\date{\today}

\begin{document}

\maketitle

% Usually we are applying learning to applications that are not completely unknown, but have some running policies. These policies are the decision rules that are currently used in the system (e.g., an engineered autonomous driving rule, or a heuristic-based system for diagnosis), and the goal of applying a learning algorithm is often to further improve upon these \emph{baseline policies}. 
% %
% As a result, it is imperative that the policy learned by the agent does not lead to \emph{performance degradation}. This criterion is especially critical for applications where the poor decision outcomes cannot be tolerated (such as health care, autonomous driving, and commercial resource allocation).

% Offline reinforcement learning (RL) is a promising framework for building real-world learning systems. 
% By optimizing for the worst-case performance, offline RL in principle can recover the best policy that an offline dataset can explain even when the data only have partial coverage. 
% %
% However, we found that the predominant principle existing offline RL algorithms can potentially lead to performance degradation for systems where there have baseline policies.
% %
% The uncertainty makes offline RL algorithms unreliable.
% %

% The potential of learning decision making policies from 

% In building reliable learning systems for real-world decision making problems, it is often critical that the learned policy does not degrade the performance of existing baseline policies. Recently the offline reinforcement learning 

\begin{abstract}

We propose a new model-based offline RL framework, called \algofull (\algo), which can robustly learn policies to improve upon an arbitrary baseline policy regardless of data coverage.
Based on the concept of relative pessimism, \algo is designed to optimize for the worst-case relative performance when facing uncertainty. 
In theory, we prove that the learned policy of \algo never degrades the performance of the baseline policy with \emph{any} admissible hyperparameter, and can learn to compete with the best policy within data coverage when the hyperparameter is well tuned, and the baseline policy is supported by the data.
Such a robust policy improvement property makes \algo especially suitable for building real-world learning systems, because in practice ensuring no performance degradation is imperative before considering any benefit learning can bring.

\end{abstract}
\section{Introduction}

% Offline RL is a promising paradigm for real-world learning for decision making, because it can learn policies from existing data regardless of its support.  Some background.
Offline reinforcement learning (RL) is a technique for learning decision-making policies from logged data~\citep{jin2021pessimism,xie2021bellman}. In comparison with alternate learning techniques, such as off-policy RL and imitation learning, offline RL reduces the data assumption needed to learn good policies and does not require collecting new data. Theoretically, offline RL can learn the best policy that the given data can explain: as long as the offline data includes all scenarios that executing a near-optimal policy would encounter, an offline RL algorithm can learn a near-optimal policy, even when the data is collected by highly sub-optimal policies or is not diverse. Such robustness to data coverage quality makes offline RL a promising technique for solving real-world problems, because collecting diverse or expert-quality data in practice is often expensive or simply infeasible.

The fundamental principle behind offline RL is the concept of pessimism in face of uncertainty, which considers worst-case outcomes for scenarios without data. In implementation, this is realized by (explicitly or implicitly) constructing performance lower bounds in policy learning, which penalizes the agent to take uncertain actions. Various designs have been proposed to construct such lower bounds, including behavior regularization~\citep{fujimoto2019off,kumar2019stabilizing,wu2019behavior,laroche2019safe,fujimoto2021minimalist}, point-wise pessimism based on negative bonuses or truncation~\citep{kidambi2020morel,jin2021pessimism}, value penalty~\citep{kumar2020conservative, yu2020mopo}, or two-player games~\citep{cheng2022adversarially,xie2021bellman,uehara2021pessimistic}.
%
% Conceptually, the tighter the lower bound is, the better the learned policy would perform, as the performance estimate is more accurate.
Conceptually, the more accurate the lower bound is, the better the learned policy would perform.
%

% However, oftentimes we do not necessarily care about the absolute performance. There is usually some policy already. We care about how much the new solution is better than this existing baseline/reference policy. We don't degradation in performance.  
Despite these advances, offline RL still has not been widely adopted to build learning-based decision systems in practice. % for real-world decision making problems. 
%, because of the robustness and reliablity of the learned results.
% Offline RL algorithm are notorious for their sensitivity to hyperparameter tuning; but since no online interaction is available, hyperparameters are selected often suboptimally, which leads to inconsistent learning outcomes.
One reason we posit is that \emph{achieving high performance in the worst case  is not the full picture of designing real-world learning agents.}
%

% , and are often those that collected the data in the first place. 
Usually, we apply machine learning to applications that are not completely unknown, but have some running policies. These policies are the decision rules that are currently used in the system (e.g., an engineered autonomous driving rule, or a heuristic-based system for diagnosis), and the goal of applying a learning algorithm is often to further improve upon these \emph{baseline policies}. 
As a result, it is imperative that the policy learned by the agent does not lead to \emph{performance degradation}. This criterion is especially critical for applications where the poor decision outcomes cannot be tolerated (such as health care, autonomous driving, and commercial resource allocation).

% But since there is missing data. The best worst-case policy does not mean it's better than the baseline always, unless the data fully covers the reference well. Show a pictorial example. Give examples why the reference policy may not be covered by the data well.  
Although optimizing for absolute or relative performance is the same when full information is available, they can lead to different policies when we only have partial data coverage. In this case, the policy that has the best worst-case performance (which most offline RL algorithms aim to recover) would not necessarily perform better than the baseline policies when deployed in the real environment. %In particular, this happens when the data do not fully cover the baseline policies' behaviors.
Such performance degradation happens when the data does not cover {all} behaviors of the baseline policies, which can be due to finite samples or a coverage mismatch between the baselines and the data collection policies.
%
% While this may sound a bit counter-intuitive, we can easily see the difference in the pictorial example in Fig \chingan{Add a figure}.
%
As a result, running policies learned by existing offline RL algorithms could risk degrading performance.

% A summary of what we do 
%introduced in the recent paper ATAC. We extend relative pessimism with the data collection to a general version with any reference. 
In this work, we propose a new model-based offline RL framework, called \algohighlighted (\algo), which can robustly learn policies improving upon an arbitrary baseline policy. 
\algo is designed based on the concept of relative pessimism~\citep{cheng2022adversarially}, which aims to optimize for the worst-case relative performance when facing uncertainty. 
In theory, we prove that the learned policy from \algo never degrades the performance of the baseline policy for a range of hyperparameters which is given beforehand, a property known as Robust Policy Improvement (RPI)~\citep{cheng2022adversarially}.
%
% In theory, we prove that the learned policy of \algo never degrades the performance of the baseline policy of comparison for a wide range of hyperparameters which are known beforehand.
%
%\algo demonstrates a stronger version of the Robust Policy Improvement (RPI) property introduced in~\citep{cheng2022adversarially}: the learned policy of \algo never degrades the performance of the baseline policy of comparison for a wide range of hyperparameters which are known beforehand, even when the baseline policy was not used to collect the data.
%
In addition, we prove that, when the right hyperparameter is chosen, and the baseline policy is covered by the data, the learned policy of \algo can also compete with any policy within data coverage in an absolute sense.

% This policy improvement property of \algo is a stronger version of Robust Policy Improvement (RPI) property introduced in~\citep{cheng2022adversarially} in the sense that

%
To our knowledge, RPI property of offline RL has so far been limited to comparing against the data collection policy (i.e. the behavior policy)~\citep{cheng2022adversarially,fujimoto2019off,kumar2019stabilizing,wu2019behavior,laroche2019safe,fujimoto2021minimalist}.
However, it is common that the baseline policy of interest is different from the behavior policy. For example, in robotics manipulation, we often have a dataset of activities different from the target task. % that we wish to solve. % and have the baseline policy for.
In this case, comparing against the behavior policy is meaningless, as these policies do not have meaningful performance in the target task.
In \algo, by using models, we extend the technique of relative pessimism to achieve RPI with \emph{arbitrary} baseline policies, regardless of whether they collected the data or not.  
%To highlight this stronger notion, we call this generalized from of RPI, strong RPI.

Finally, based on RPI, we discuss and compare different solution concepts for offline RL (such as relative pessimism here as well as other approaches like absolute pessimism and minimax regret).
We show that while these concepts are the same in online RL, in general they lead to different results in offline RL because of the undiminishable uncertainty due to missing data coverage. %, while each of them is optimal to its own right.
Our discussion reveals some interesting observations and important implications to offline RL algorithm design, which we feel that many in the offline RL community are not actively aware of.

\vspace{-2mm}
\section{Preliminaries}
\vspace{-1mm}

\paragraph{Markov Decision Process}
We consider an agent acting in an infinite-horizon discounted Markov Decision Process (MDP) $M$ defined by the tuple  $\langle \mathcal{S}, \mathcal{A}, \mathcal{P}, R,  \gamma \rangle$ where $\mathcal{S}$ is the state space, $\mathcal{A}$ is the action space, $\mathcal{P} : \mathcal{S} \times \mathcal{A} \rightarrow \Delta\left(\mathcal{S}\right)$ is the transition dynamics, $R: \mathcal{S} \times \mathcal{A} \rightarrow \left[0, 1 \right]$ is a scalar reward function %(over this paper we set $\Rmax = 1$ for simplicity) 
and $\gamma \in [0, 1)$ is the discount factor. The learner selects actions using a policy $\pi : \mathcal{S} \rightarrow \Delta\left(\mathcal{A}\right)$. We denote by $\Pi$ the space of all Markovian policies. Let, $d_M^\pi(s,a)$ denote the discounted state-action distribution obtained by running policy $\pi$ on $M$, i.e $d_M^\pi(s,a) = \left(1 - \gamma\right)\mathbb{E}\left[\sum_{t=0}^{\infty} \gamma^t \mathbbm{1} \left(s_t = s, a_t = a | a_t \sim \pi \left(\cdot | s_t\right) \right) \right] $. Let $J_M(\pi)= \mathbb{E}_{\pi, M}\left[ \sum_{t=0}^{\infty} \gamma^t r_t | a_t \sim \pi \right]$ be the expected discounted return of policy $\pi$ on $M$. The goal of reinforcement learning is to find the policy that maximizes $J$. We define the value function as $V^\pi_M(s) = \mathbb{E}_{\pi, M}\left[\sum_{t=0}^{\infty} \gamma^t r_t | s_0 = s \right] $, and the related state-action value function (i.e., Q-function) as $Q^\pi_M(s,a) = \mathbb{E}_{\pi, M}\left[\sum_{t=0}^{\infty} \gamma^t r_t  |s_0 = s, s_0 = a\right]$. We use $[0,\Vmax]$ as the range of value functions.

\vspace{-1mm}
\paragraph{Offline RL}

The aim of offline RL is to output strong policies from a fixed dataset collected using a behavior policy without further environmental interactions. We assume the dataset $\mathcal{D}$ consists of $\lbrace \left(s_i, a_i, r_i, s_{i+1}\right) \rbrace_{i=1}^{N}$, where $(s_i,a_i)$ is sampled i.i.d.~from some distribution $\mu$. % $\forall i \in [N]$.
We also abuse $\mu$ as discounted state-action occupancy of behavior policy, i.e., $\mu \equiv d^\mu_M$, and we use $a \sim \mu(\cdot | s)$ to denote sampling from that behavior policy.
This paper is concerned with the model-based offline RL problem, and we use $\Mcal$ to denote the model class. For each $M \in \Mcal$, we use $P_M :\Scal \times \Acal \to \Delta(\Scal)$ and $R_M:\Scal \times \Acal \to [0,\Rmax]$ to denote the corresponding transition and reward function of $M$.
\begin{assumption}[Realizability]
\label{asm:realizability}
We assume the ground truth model $M^\star$ is in the model class $\Mcal$.%, $M^\star \in \Mcal$.
\end{assumption}

%%%%%%%%%%%%%%%%%%%%%%%%%%%%%%%%%%%%%%%%%%%%%%%%%%%%%%%%%%
\vspace{-2mm}
\section{\algofull (\algo)} \label{sec:algo}
\vspace{-1mm}

In this section, we introduce our proposed approach, \algofull (\algo), in  \cref{alg:atmo}, and present the main theoretical results.
\algo can be viewed as a model-based extension of the ATAC algorithm by~\citet{cheng2022adversarially}. In the next sections, we illustrate that \algo is not only able to compete with the best data-covered policy as prior works~\citep[e.g.,][]{xie2021bellman,uehara2021pessimistic,cheng2022adversarially}, but also enjoys a stronger robust policy improvement guarantee than \citep{cheng2022adversarially}.

\begin{algorithm}[t]
\caption{\algohighlighted (\algo)}
\label{alg:atmo}
{\bfseries Input:} Batch data $\Dcal$. Model class $\Mcal$. Coefficient $\alpha$. Policy class $\Pi$. Reference policy $\piref$.
\begin{algorithmic}[1]
\State Construct version space for the model,
\begin{gather}
\label{eq:v_space}
\Mcal_\alpha = \{M \in \Mcal: \max_{M' \in \Mcal} \Lcal_{\Dcal}(M') - \Lcal_\Dcal(M) \leq \alpha \},
\\
\label{eq:def_loss}
\text{where \quad } \textstyle \Lcal_\Dcal(M) \coloneqq \sum_{(s,a,r,s') \in \Dcal} \left[ \log \PP_M(s'|s,a) - \left(R_M(s,a) - r\right)^2 \right],~\forall M \in \Mcal.
\end{gather}
\State Conduct learning via relative pessimism,
\begin{align}
\label{eq:def_pihat}
\pihat = \argmax_{\pi \in \Pi} \min_{M \in \Mcal_\alpha} J_M(\pi) - J_M(\piref).
\end{align}
%
% \State Output $\Mcal_\alpha$ (for replanning) and $\pihat$.
\end{algorithmic}
\end{algorithm}

% \begin{algorithm}[th]
% \caption{Online Replanning with \algo}
% \label{alg:replanning}
% {\bfseries Input:} Version space of models $\Mcal_\alpha$. Policy class $\Pi$. Reference policy $\piref$.
% \begin{algorithmic}[1]
% \State Initialize $\pi_0 = \piref$.
% \For{$t = 1,2,\dotsc$}
% \State Observe state $s_t$.
% \State Compute $\displaystyle \pi_t = \argmax_{\pi \in \Pi} \min_{M \in \Mcal_\alpha} V_M^{\pi}(s_t) - V_M^{\pi_{t-1}}(s_t)$.
% \State Act $a_t \sim \pi_t(\cdot|s_t)$.
% \EndFor
% \end{algorithmic}
% \end{algorithm}

%\subsection{Theoretical Analysis}
Below we analyze \algo theoretically and present guarantees on its absolute performance and the policy improvement over the reference policy $\piref$. Before presenting the detailed guarantees, we introduce generalized single-policy concentrability, which measures the distribution shift over some arbitrary policy $\pi$ and data distribution $\mu$.

\begin{definition}[Generalized Single-policy Concentrability]
We define the generalized single-policy concentrability for policy $\pi$ for model class $\Mcal$ and offline data distribution $\mu$ as
\begin{small}
\begin{align*}
\C_\Mcal(\pi) \coloneqq \sup_{M \in \Mcal} \frac{\E_{d^\pi} \left[ \TV\left(P_M(\cdot|s,a), P^\star(\cdot|s,a)\right)^2 + \left(R_M(s,a) - R^\star(s,a)\right)^2 \right] }{d_{\mu} \left[ \TV\left(P_M(\cdot|s,a), P^\star(\cdot|s,a)\right)^2 + \left(R_M(s,a) - R^\star(s,a)\right)^2 \right] }.
\end{align*}
\end{small}
\end{definition}
Note that $\C_\Mcal(\pi)$ is always upper bounded by the standard single-policy concentrability coefficient $\| d^\pi / \mu \|_\infty$~\citep[e.g.,][]{jin2021pessimism, rashidinejad2021bridging,xie2021policy}, but it can be smaller in general with model class $\Mcal$. It can also be viewed as a model-based analog of the one in~\citet{xie2021bellman}, and the detailed discussion around $\C_\Mcal(\pi)$ refers to~\citet{uehara2021pessimistic}.

We are now ready to present the absolute performance guarantee of \algo.

\begin{theorem}[Absolute performance]
\label{thm:perf}
Under \cref{asm:realizability}, there is an absolute constant $c$ such that for any $\delta \in (0,1]$, if we set $\alpha = c\cdot(\log(\nicefrac{|\Mcal|}{\delta}))$ in \cref{alg:atmo}, then for any reference policy $\piref$ and comparator policy $\picom \in \Pi$, with probability $1 - \delta$, the policy $\pihat$ of \cref{alg:atmo} satisfies
\begin{align*}
J(\picom) - J(\pihat) \leq &~ \Ocal \left( \left[\sqrt{\C_\Mcal(\picom)} + \sqrt{\C_\Mcal(\piref)}\right] \cdot \frac{\Vmax}{1 - \gamma} \sqrt{\frac{\log(\nicefrac{|\Mcal|},{\delta})}{n}} \right).
\end{align*}
\end{theorem}
Roughly speaking, \cref{thm:perf} shows that $\pihat$ learned by \cref{alg:atmo} could compete with any policy $\picom$ with a large enough dataset, as long as the offline data $\mu$ has good coverage on $\picom$ (since the reference policy $\piref$ is the input of \cref{thm:perf}, one can set $\piref = \mu$ (data collection policy) as $\C_\Mcal(\mu) \leq \C_\Mcal(\picom)$).
Compared to the closest model-based offline RL work~\citep{uehara2021pessimistic}, if we set $\piref = \mu$ (data collection policy), \cref{thm:perf} leads to almost the same guarantee as~\citet[Theorem 1]{uehara2021pessimistic} (up to constant factors).

In addition to the guarantee on the absolute performance, below we show that, if \cref{asm:realizability} is satisfied and $\piref \in \Pi$, \algo always improves over $J(\piref)$ for \emph{a wide range choice of pessimistic parameter $\alpha$}. Compared with the model-free ATAC algorithm in \citep[Prop. 6]{cheng2022adversarially}, \cref{thm:RPI} removes the concentration errors of $O(\sqrt{1/N})$ as \algo is model-based.

\begin{theorem}[Robust strong policy improvement]
\label{thm:RPI}
Under \cref{asm:realizability}, there exists an absolute constant $c$ such that for any $\delta \in (0,1]$, if: i) $\alpha \geq c\cdot(\log(\nicefrac{|\Mcal|}{\delta}))$ in \cref{alg:atmo}; ii) $\piref \in \Pi$, then with probability $1 - \delta$, the policy $\pihat$ learned by \cref{alg:atmo} satisfies $J(\piref) \leq J(\pihat)$. 
\end{theorem}

%%%%%%%%%%%%%%%%%%%%%%%%%%%%%%%%%%%%%%%%%%%%%%%%%%%%%%%%%%
\vspace{-2mm}
\section{Robust Policy Improvement (RPI)} \label{sec:RPI}
\vspace{-2mm}

\subsection{How to formally define RPI?}
\vspace{-2mm}

Improving over some reference policy has been long studied in the literature. %via various algorithmic concepts. 
To highlight the advantage of \algo, we formally give the definition of different  policy improvement properties.

\begin{definition}[Robust policy improvement] \label{def:RPI}
Suppose $\pihat$ is the learned policy from an algorithm. We say the algorithm has the policy improvement (PI) guarantee if $J(\piref) - J(\pihat) \leq \nicefrac{o(N)}{N}$ is guaranteed for some reference policy $\piref$ with offline data $\Dcal \sim \mu$, where $N =|\Dcal|$.
We use the following two criteria w.r.t.~$\piref$ and $\mu$ to define different kinds PI:
\begin{enumerate}[(i)]
\item The PI is \txieul{strong} if $\piref$ can be selected arbitrarily from policy class $\Pi$ regardless of the choice data-collection policy $\mu$; otherwise, PI is \txieul{weak} (i.e., $\piref \equiv \mu$ is required).
\item The PI is \txieul{robust} if it can be achieved by a range of hyperparameters with a known subset.
\end{enumerate}
\end{definition}

% The definition of policy improvement is motivated by the practical scenario that we wish to offline learn policies that are strictly better than a reference policy, where the reference policy is usually an existing policy for the application.
Weak policy improvement is also known as \emph{safe policy improvement} in the literature~\citep{fujimoto2019off,laroche2019safe}. It requires the reference policy to be also the behavior policy that collects the offline data. 
In comparison, strong policy improvement imposes a stricter requirement, which requires policy improvement \emph{regardless} of how the data were collected. This condition is motivated by the common situation where the reference policy is not the data collection policy. 
%For example, in a multi-task problem with shared dynamics, the data are collected by policies for different tasks, and the reference policy we wish to improve on is task specific. In this case, weak policy improvement is meaningless because the behavior policy, which is the average of policies from all tasks, does not have meaningful performance in the target task.
%
Finally, since we are learning policies offline, without online interactions, it is not straightforward to tune the hyperparameter directly. 
Therefore, it is desirable that we can design algorithms with these properties in a robust manner in terms of hyperparameter selection. Formally, \cref{def:RPI} requires the policy improvement to be achievable by a set of hyperparameters that is known before learning.

\cref{thm:RPI} indicates the robust strong policy improvement of \algo. On the other hand, algorithms with robust weak policy improvement are available in the literature~\citep{cheng2022adversarially,fujimoto2019off,kumar2019stabilizing,wu2019behavior,laroche2019safe,fujimoto2021minimalist}; this is usually achieved by designing the algorithm to behave like imitation learning (IL) for a known set of hyperparameter (e.g.,  behavior regularization algorithms have a weight that can turn off the RL behavior and regress to IL). 
However, deriving guarantees of achieving the best data-covered policy of the IL-like algorithm is challenging due to its imitating nature. To our best knowledge, ATAC~\citep{cheng2022adversarially} is the only algorithm that achieves both robust (weak) policy improvement as well as guarantees absolute performance.

\vspace{-2mm}
\subsection{When RPI actually improves?}
\vspace{-2mm}

Given \algo's ability to improve over an arbitrary policy, the following questions naturally arise:
%\begin{center}\it 
\textit{
Can \algo nontrivially improve the output policy of other algorithms (e.g., such as those based on \emph{absolute pessimism} \citep{xie2021bellman}), including itself?}
%\end{center}
Note that outputting $\piref$ itself always satisfies RPI, but such result is trivial.  By ``nontrivially'' we mean a non-zero worst-case improvement. %; otherwise, the statement would be meaningless. 
If the statement were true, we would be able to repeatedly run \algo to improve over itself and then obtain the \emph{best} policy any algorithm can learn offline.
%we would be able to repeatedly run \algo to improve over its own output with more iterations and learn the \emph{best} offline policy.

Unfortunately, the answer is negative. Not only \algo cannot improve over itself, but it also cannot improve over a variety of algorithms. In fact, the optimal policy of an \textit{arbitrary} model in the version space is unimprovable (see Corollary~\ref{th:fixed-point examples})! Our discussion reveals some interesting observations (e.g., how equivalent performance metrics for online RL can behave very differently in the offline setting) and their implications (e.g., how we should choose $\piref$ for \algo). Despite their simplicity, we feel that many in the offline RL community are not actively aware of these facts (and the unawareness has led to some confusion), which we hope to clarify below.

\paragraph{Setup} We consider an abstract setup where the learner is given a version space $\Mcal_\alpha$ that contains the true model %($M^\star \in \Mcal_\alpha$), 
and needs to choose a policy $\pi\in\Pi$ based on $\Mcal_\alpha$. We use the same notation $\Mcal_\alpha$ as before, but emphasize that it does not have to be constructed as in \eqref{eq:v_space} and \eqref{eq:def_loss}. 
In fact, for the purpose of this discussion, the data distribution,  sample size, data randomness, and estimation procedure for constructing $\Mcal_\alpha$ are \textbf{all irrelevant}, as our focus here is how decisions should be made with a given $\Mcal_\alpha$. This makes our setup very generic and the conclusions widely applicable.

To facilitate discussion, we define the \textit{fixed point} of \algo's relative pessimism step:
\begin{definition} \label{def:fixed_point}
Consider \eqref{eq:def_pihat} as an operator that maps an arbitrary policy $\piref$ to $\pihat$. A fixed point of this \emph{relative pessimism} operator is, therefore, any policy $\pi \in \Pi$ such that
$
\pi \in \argmax_{\pi' \in \Pi} \min_{M \in \Mcal_\alpha} J_M(\pi') - J_M(\pi)
$.
\end{definition}
Given the definition, relative pessimism cannot improve over a policy if it is already a fixed point. Below we show a sufficient and necessary condition for being a fixed point, and show a number of concrete examples (some of which may be surprising) that are fixed points and thus unimprovable.

\begin{lemma}[Fixed-point Lemma] \label{lm:fixed-point lm of mb-atac}
For any $\Mcal \subseteq \Mcal_\alpha$ and any $\psi:\Mcal\to\mathbb{R}$, consider the policy
\begin{align} \label{eq:candidate fixed-point policies}
    \pi \in \argmax_{\pi'\in\Pi} \min_{M\in\Mcal} J_M(\pi') + \psi(M)
\end{align}
Then $\pi$ is a fixed point in Definition~\ref{def:fixed_point}.
Conversely, for any fixed point $\pi$ in \cref{def:fixed_point}, there is a $\psi:\Mcal\to\mathbb{R}$ such that $\pi$ is a solution to \eqref{eq:candidate fixed-point policies}.
\end{lemma}
% \begin{proof}
% We prove the result by contradiction.
% First notice $\min_{M\in\Mcal} J_M(\pi') - J_M(\pi') = 0$.
% Suppose there is $\overline{\pi} \in \Pi$ such that $\min_{M\in\Mcal_\alpha} J_M(\bar{\pi}) - J_M(\pi') >0$, which implies that $J_M(\bar{\pi}) >J_M(\pi') $, $\forall M \in \Mcal_\alpha$.
% %Let $\overline{M} = \argmin_{M\in\Mcal} J_M(\bar{\pi}) - J_M(\pi')$. This would imply that 
% Since $\Mcal \subseteq \Mcal_\alpha$, we have 
% \begin{align*}
%     \min_{M\in\Mcal} J_M(\bar{\pi})  + \psi(M) 
%     &>  \min_{M\in\Mcal}  J_M(\pi')  + \psi(M) =  \max_{\pi\in\Pi}\min_{M\in\Mcal}  J_M(\pi)  + \psi(M)
% \end{align*}
% which is a contradiction of the maximin optimality. Thus $\max_{\pi\in\Pi} \min_{M\in\Mcal_\alpha} J_M(\bar{\pi}) - J_M(\pi') = 0 $, which means $\pi'$ is a solution.

% For the converse statement, suppose $\pi$ is a fixed point. We can just let $\psi(M) = -J_M(\pi)$. Then this pair of $\pi$ and $\psi$ by definition of the fixed point satisfies \eqref{eq:candidate fixed-point policies}.
% \end{proof}

%We can view that function $\psi$ as a regularization term on the MDP player. We can use \eqref{eq:candidate fixed-point policies} to instantiate different algorithms and \cref{lm:fixed-point lm of mb-atac} shows that the solution to these algorithms are fixed points of relative pessimism. In particular, this result implies that we cannot iterate relative pessimism forever; it would converge in one iteration.
\begin{corollary} \label{th:fixed-point examples}
The following are fixed points of relative pessimism (Definition~\ref{def:fixed_point}):
\begin{enumerate}
    \item  Absolute-pessimism policy, i.e., $\psi(M) = 0 $.
    \item Relative-pessimism policy for any reference policy, i.e., $\psi(M) = - J_M(\piref)$. 
    \item Regret-minimization policy, i.e., $\psi(M) = -J_M(\pi_M^*)$, where $\pi_M^* \in \argmax_{\pi\in\Pi} J_M(\pi)$. % is an optimal policy for $M$.
    \item Optimal policy of an \emph{arbitrary} model $M \in \Mcal_\alpha$, $\pi_M^*$, i.e., $\Mcal = \{M\}$. 
    This would include the optimistic policy, that is, $\argmax_{\pi\in\Pi, M\in\Mcal_\alpha} J_M(\pi)$
    %\nan{UCB $\Rightarrow$ optimistic? also, define and add citations?}
\end{enumerate}
\end{corollary}

% Based on this result, we now have a discussion about absolute pessimism (\#1) vs.~regret minimization (\#3), and then extend the lesson to the more general setting. 
\vspace{-2mm}
\paragraph{Return maximization and regret minimization are \textit{different} in offine RL} We first note that these four examples generally produce different policies, even though some of them optimize for objectives that are traditionally viewed as equivalent in online RL (the ``worst-case over $\Mcal_\alpha$'' part of the definition does not matter in online RL), e.g., absolute pessimism optimizes for $J_M(\pi)$, which is the same as minimizing the regret $J_M(\pi_M^\star) - J_M(\pi)$ for a fixed $M$. However, their  equivalence in online RL relies on the fact that online exploration can eventually resolve any model uncertainty when needed, so we only need to consider the performance metrics w.r.t.~the true model $M=M^\star$. In offline RL with an arbitrary data distribution (since we do not make any coverage assumptions), there will generally be model uncertainty that cannot be resolved, and worst-case reasoning over such model uncertainty (i.e., $\Mcal_\alpha$) separates apart the definitions that are once equivalent. 

Moreover, it is impossible to  compare  return maximization and regret minimization and make a claim about which one is better. They are not simply an algorithm design choice, but are definitions of the learning goals and the guarantees themselves---thus incomparable: if we care about obtaining a guarantee for the worst-case \textit{return}, the  return maximization is optimal by definition; if we are more interested in obtaining a guarantee for the worst-case \textit{regret}, then again, regret minimization is trivially optimal. We also note that analyzing algorithms under a metric that is different from the one they are designed for can lead to unusual conclusions. For example, \citet{xiao2021optimality} show that optimistic/neutral/pessimistic algorithms\footnote{Incidentally, optimistic/neutral policies correspond to \#4 in Corollary~\ref{th:fixed-point examples}.} are equally minimax-optimal in terms of their regret guarantees in offline multi-armed bandits. However, the algorithms they consider are optimistic/pessimistic w.r.t.~the return---as commonly considered in the offline RL literature---not w.r.t.~the regret which is the performance metric they are interested in analyzing. 

\vspace{-2mm}
\paragraph{$\piref$ is more than a hyperparameter---it defines the performance metric and learning goal} Corollary~\ref{th:fixed-point examples}  shows that \algo (with relative pessimism) has many different fixed points, some of which may seem quite unreasonable for offline learning, such as greedy w.r.t.~an arbitrary model or even optimism (\#4). From the above discussion, we can see that this is not a defect of the algorithm. Rather, in the offline setting with unresolvable model uncertainty, there are many different performance metrics/learning goals that are generally incompatible/incomparable with each other, and the agent designer must make a choice among them and convey the choice to the algorithm. In \algo, such a choice is explicitly conveyed by the choice of $\piref$, which subsumes return maximization and regret minimization as special cases (\#2 and \#3 in Corollary~\ref{th:fixed-point examples}).

\bibliographystyle{plainnat}
\bibliography{ref}

%%%%%%%%%%%%%%%%%%%%%%%%%%%%%%%%%%%%%%%%%%%%%%%%%%%%%%%%%%

% Appendix

\clearpage

\appendix
\onecolumn

\begin{center}
{\LARGE Appendix}
\end{center}

%%%%%%%%%%%%%%%%%%%%%%%%%%%%%%%%%%%%%%%%%%%%%%%%%%%%%%%%%%

% \section{On the Fixed Points of Robust Policy Improvement}
% \input{fixed_point}

\section{Proofs for \creftitle{sec:algo}}

\subsection{Technical Tools}

\begin{lemma}[Simulation lemma]
\label{lem:simulation}
Consider any two MDP model $M$ and $M'$, and any $\pi: \Scal \to \Delta(\Acal)$, we have
\begin{gather*}
\left| J_M(\pi) - J_{M'}(\pi) \right| \leq \frac{\Vmax}{1 - \gamma} \E_{d^\pi} \left[ \TV\left(P_M(\cdot|s,a), P_{M'}(\cdot|s,a)\right) \right] + \frac{1}{1 - \gamma} \E_{d^\pi} \left[ \left|R_M(s,a) - R_{M'}(s,a)\right| \right].
\end{gather*}
\end{lemma}
\cref{lem:simulation} is the standard simulation lemma in model-based reinforcement learning literature, and its proof can be found in, e.g., \citet[Lemma 7]{uehara2021pessimistic}.

\subsection{Guarantees about Version Space}

\begin{lemma}
\label{lem:Vspace_Mstar}
Let $M^\star$ be the ground truth model. Then, with probability at least $1 - \delta$, we have
\begin{align*}
\max_{M \in \Mcal} \Lcal_{\Dcal}(M) - \Lcal_{\Dcal}(M^\star) \leq \Ocal \left( \log(\nicefrac{|\Mcal|}{\delta}) \right),
\end{align*}
where $\Lcal_\Dcal$ is defined in \cref{eq:def_loss}.
\end{lemma}
\begin{proof}[\cpfname{lem:Vspace_Mstar}]
By \cref{lem:MLE_Mstar}, we know
\begin{align}
\label{eq:Vspace_Mstar_P}
\max_{M \in \Mcal} \log \ell_{\Dcal}(M) - \log \ell_{\Dcal}(M^\star) \leq \log(\nicefrac{|\Mcal|}{\delta}).
\end{align}
In addition, by~\citet[Theorem A.1]{xie2021bellman} (with setting $\gamma = 0$), we know w.p.~$1-\delta$,
\begin{align}
\label{eq:Vspace_Mstar_R}
\sum_{(s,a,r,s') \in \Dcal} \left(R^\star(s,a) - r\right)^2 - \min_{M \in \Mcal} \sum_{(s,a,r,s') \in \Dcal} \left(R_M(s,a) - r\right)^2 \lesssim \log(\nicefrac{|\Mcal|}{\delta}).
\tag{$\Rmax$}
\end{align}
Combining the \cref{eq:Vspace_Mstar_P,eq:Vspace_Mstar_R}, we have w.p.~$1-\delta$,
\begin{align*}
&~ \max_{M \in \Mcal} \Lcal_{\Dcal}(M) - \Lcal_{\Dcal}(M^\star)
\\
\leq &~ \max_{M \in \Mcal} \log \ell_{\Dcal}(M) - \min_{M \in \Mcal} \sum_{(s,a,r,s') \in \Dcal} \left(R_M(s,a) - r\right)^2 - \Lcal_{\Dcal}(M^\star)
\\
\lesssim &~ \log(\nicefrac{|\Mcal|}{\delta}).
\end{align*}
This completes the proof.
\end{proof}

\begin{lemma}
\label{lem:Vspace_M}
For any $M \in \Mcal$, we have with probability at least $1 - \delta$,
\begin{gather*}
\E_\mu \left[ \TV\left(P_M(\cdot|s,a), P^\star(\cdot|s,a)\right)^2 + \left(R_M(s,a) - R^\star(s,a)\right)^2 \right]
\\
\leq \Ocal \left( \frac{\max_{M' \in \Mcal} \Lcal_{\Dcal}(M') - \Lcal_\Dcal(M) + \log(\nicefrac{|\Mcal|}{\delta})}{n} \right),
\end{gather*}
where $\Lcal_\Dcal$ is defined in \cref{eq:def_loss}.
\end{lemma}
\begin{proof}[\cpfname{lem:Vspace_M}]
By \cref{lem:MLE_bound_TV}, we have w.p.~$1-\delta$,
\begin{align}
\label{lem:Vspace_MP}
n \cdot \E_\mu \left[ \TV\left(P_M(\cdot|s,a), P^\star(\cdot|s,a)\right)^2 \right] \lesssim \log\ell_{\Dcal}(M^\star) - \log\ell_\Dcal(M) + \log(\nicefrac{|\Mcal|}{\delta}).
\end{align}
Also, we have
\begin{align}
\label{lem:Vspace_MR}
&~ n \cdot \E_\mu \left[ \left(R_M(s,a) - R^\star(s,a)\right)^2 \right]
\\
\nonumber
= &~ n \cdot \E_\mu \left[ \left(R_M(s,a) - r \right)^2 \right] -  n \cdot \E_\mu \left[ \left(R^\star(s,a) - r \right)^2\right]
\tag*{\citep[see, e.g.,][Eq.\;(A.10) with $\gamma = 0$]{xie2021bellman}}
\\
\nonumber
\lesssim &~ \sum_{(s,a,r,s') \in \Dcal} \left(R_M(s,a) - r\right)^2 - \sum_{(s,a,r,s') \in \Dcal} \left(R^\star(s,a) - r\right)^2 + \log(\nicefrac{|\Mcal|}{\delta}),
\end{align}
where the last inequality is a direct implication of \citet[Lemma A.4]{xie2021bellman} and $\Rmax = 1$. Combining \cref{lem:Vspace_MP,lem:Vspace_MR}, we obtain
\begin{align*}
&~ n \cdot \E_\mu \left[ \TV\left(P_M(\cdot|s,a), P^\star(\cdot|s,a)\right)^2 + \left(R_M(s,a) - R^\star(s,a)\right)^2 \right]
\\
\lesssim &~ \log\ell_{\Dcal}(M^\star) - \sum_{(s,a,r,s') \in \Dcal} \left(R^\star(s,a) - r\right)^2 - \log\ell_\Dcal(M) + \sum_{(s,a,r,s') \in \Dcal} \left(R_M(s,a) - r\right)^2 + \log(\nicefrac{|\Mcal|}{\delta})
\\
= &~ \Lcal_{\Dcal}(M^\star) - \Lcal_\Dcal(M) + \log(\nicefrac{|\Mcal|}{\delta})
\\
\leq &~ \max_{M' \in \Mcal}\Lcal_{\Dcal}(M') - \Lcal_\Dcal(M) + \log(\nicefrac{|\Mcal|}{\delta}).
\end{align*}
This completes the proof.

\end{proof}

\subsection{MLE Guarantees}

We use $\ell_\Dcal(M)$ to denote the likelihood of model $M = (P,R)$ with offline data $\Dcal$, where
\begin{align}
\label{eq:def_LD}
\mathfrak \ell_\Dcal(M) = &~ \prod_{(s,a,r,s') \in \Dcal} P_M(s'|s,a).
\end{align}

For the analysis around maximum likelihood estimation, we largely follow the proving idea of~\citet{agarwal2020flambe,liu2022partially}, which is inspired by~\citet{zhang2006eps}.

The next lemma shows that the ground truth model $M^\star$ has a comparable log-likelihood compared with MLE solution.
\begin{lemma}
\label{lem:MLE_Mstar}
Let $M^\star$ be the ground truth model. Then, with probability at least $1 - \delta$, we have
\begin{align}
\label{eq:mle_lemma}
\max_{M \in \Mcal} \log \ell_{\Dcal}(M) - \log \ell_{\Dcal}(M^\star) \leq \log(\nicefrac{|\Mcal|}{\delta}).
\end{align}
\end{lemma}
\begin{proof}[\cpfname{lem:MLE_Mstar}]

The proof of this lemma is obtained by a standard argument of MLE~\citep[see, e.g.,][]{van2000empirical}. For any $M \in \Mcal$,
\begin{align}
\nonumber
\E \left[ \exp \left( \log \ell_{\Dcal}(M) - \log \ell_{\Dcal}(M^\star) \right) \right]
= &~ \E \left[ \frac{\ell_{\Dcal}(M)}{\ell_{\Dcal}(M^\star)} \right]
\\
\nonumber
= &~ \E \left[ \frac{\prod_{(s,a,r,s') \in \Dcal} \PP_{M}(s'|s,a)}{\prod_{(s,a,r,s') \in \Dcal} \PP_{M^\star}(s'|s,a)} \right]
\\
\nonumber
= &~ \E \left[ \prod_{(s,a,r,s') \in \Dcal} \frac{\PP_{M}(s'|s,a)}{\PP_{M^\star}(s'|s,a)} \right]
\\
\nonumber
= &~ \E \left[ \prod_{(s,a) \in \Dcal} \E \left[ \frac{\PP_{M}(s'|s,a)}{\PP_{M^\star}(s'|s,a)} \midmid s,a \right] \right]
\\
\nonumber
= &~ \E \left[ \prod_{(s,a) \in \Dcal} \sum_{s',r} \PP_{M}(s'|s,a)  \right]
\\
\label{eq:markov_eq_exp}
= &~ 1.
\end{align}
Then by Markov's inequality, we obtain
\begin{align*}
&~ \PP\left[\left( \log \ell_{\Dcal}(M) - \log \ell_{\Dcal}(M^\star) \right) > \log(\nicefrac{1}{\delta})\right]
\\
\leq &~ \underbrace{\E \left[ \exp \left( \log \ell_{\Dcal}(M) - \log \ell_{\Dcal}(M^\star) \right) \right]}_{=1\text{ by \Eqref{eq:markov_eq_exp}}} \cdot \exp\left[ - \log(\nicefrac{1}{\delta}) \right] = \delta.
\end{align*}
Therefore, taking a union bound over $\Mcal$, we obtain
\begin{align*}
\PP\left[\left( \log \ell_{\Dcal}(M) - \log \ell_{\Dcal}(M^\star) \right) > \log(\nicefrac{|\Mcal|}{\delta}) \right] \leq \delta.
\end{align*}
This completes the proof.
\end{proof}

% \tx{idea (maybe not very related): if we assume the model class is Lipschitz, can we ``sample'' version space via splitting data to $m$ sets, minimizing their model-fitting error to obtain, and approximating version space by those $m$ models?}

The following lemma shows that, the on-support error of any model $M \in \Mcal$ can be captured via its log-likelihood (by comparing with the MLE solution).

\begin{lemma}
\label{lem:MLE_bound_TV}
For any $M = (P,R)$, we have with probability at least $1 - \delta$,
\begin{align*}
\E_\mu \left[ \TV\left(P(\cdot|s,a), P^\star(\cdot|s,a)\right)^2 \right] \leq \Ocal \left( \frac{\log\ell_{\Dcal}(M^\star) - \log\ell_\Dcal(M) + \log(\nicefrac{|\Mcal|}{\delta})}{n} \right),
\end{align*}
where $\ell_{\Dcal}(\cdot)$ is defined in \Eqref{eq:def_LD}.
\end{lemma}
\begin{proof}[\cpfname{lem:MLE_bound_TV}]
By~\citet[Lemma 25]{agarwal2020flambe}, we have
\begin{align}
\label{eq:P_tvnorm}
\E_\mu \left[ \TV\left(P(\cdot|s,a), P^\star(\cdot|s,a)\right)^2 \right] \leq &~ -2 \log \E_{\mu \times P^\star} \left[ \exp \left( - \frac{1}{2} \log \left( \frac{P^\star(s'|s,a)}{P(s'|s,a)} \right)\right) \right]
\\
\nonumber
\E_\mu \left[ \TV\left(R(\cdot|s,a), R^\star(\cdot|s,a)\right)^2 \right] \leq &~ -2 \log \E_{\mu \times R^\star} \left[ \exp \left( - \frac{1}{2} \log \left( \frac{R^\star(r|s,a)}{R(r|s,a)} \right)\right) \right],
\end{align}
where $\mu \times P^\star$ and $\mu \times R^\star$ denote the ground truth offline joint distribution of $(s,a,s')$ and $(s,a,r)$.

Let $\wt\Dcal = \{(\wt s_i,\wt a_i,\wt r_i,\wt s_i')\}_{i = 1}^{n} \sim \mu$ be another offline dataset that is independent to $\Dcal$. Then,
\begin{align}
\nonumber
&~ - n \cdot \log \E_{\mu \times P^\star} \left[ \exp \left( - \frac{1}{2} \log \left( \frac{P^\star(s'|s,a)}{P(s'|s,a)} \right)\right) \right]
\\
\nonumber
= &~ - \sum_{i = 1}^{n} \log \E_{(\wt s_i,\wt a_i,\wt s_i') \sim \mu} \left[ \exp \left( - \frac{1}{2} \log \left( \frac{P^\star(\wt s_i'|\wt s_i,\wt a_i)}{P(\wt s_i'|\wt s_i,\wt a_i)} \right)\right) \right]
\\
\nonumber
= &~ -\log \E_{\wt\Dcal \sim \mu} \left[ \exp \left( \sum_{i = 1}^{n}  - \frac{1}{2} \log \left( \frac{P^\star(\wt s_i'|\wt s_i,\wt a_i)}{P(\wt s_i'|\wt s_i,\wt a_i)} \right) \right) \midmid \Dcal \right]
\\
\label{eq:P_tvnorm_ub}
= &~ -\log \E_{\wt\Dcal \sim \mu} \left[ \exp \left( \sum_{(s,a,s') \in \wt\Dcal}  - \frac{1}{2} \log \left( \frac{P^\star(s'|s,a)}{P(s'|s,a)} \right) \right) \midmid \Dcal \right].
\end{align}
We use $\ell_P(s,a,s')$ as the shorthand of $- \frac{1}{2} \log \left( \frac{P^\star(s| s,a)}{P( s'| s, a)} \right)$, for any $(s,a,s') \in \Scal \times \Acal \times \Scal$.
By~\citet[Lemma 24]{agarwal2020flambe}~\citep[see also][Lemma 15]{liu2022partially}, we know
\begin{align*}
\E_{\Dcal\sim\mu}\left[\exp \left( \sum_{(s,a,s') \in \Dcal}  \ell_P(s, a, s') - \log \E_{\wt\Dcal \sim \mu} \left[ \exp \left( \sum_{(s,a,s') \in \wt\Dcal}  \ell_P(s, a, s') \right) \midmid \Dcal \right] - \log|\Mcal|\right)\right] \leq 1.
\end{align*}

Thus, we can use Chernoff method as well as a union bound on the equation above to obtain the following exponential tail bound: with probability at least $1 - \delta$, we have for all $(P,R) = M \in \Mcal$,
\begin{align}
\label{eq:P_tvnorm_tail}
- \log \E_{\wt\Dcal \sim \mu} \left[ \exp \left( \sum_{(s,a,s') \in \wt\Dcal}  \ell_P(s, a, s') \right) \midmid \Dcal \right] \leq -\sum_{(s,a,s') \in \Dcal}  \ell_P(s, a, s') + 2\log(\nicefrac{|\Mcal|}{\delta}).
\end{align}
Plugging back the definition of $\ell_P$ and combining \Eqref{eq:P_tvnorm,eq:P_tvnorm_ub,eq:P_tvnorm_tail}, we obtain
\begin{align}
\label{eq:Pmle_final}
n \cdot \E_\mu \left[ \TV\left(P(\cdot|s,a), P^\star(\cdot|s,a)\right)^2 \right] \leq &~ \frac{1}{2} \sum_{(s,a,s') \in \Dcal}  \log \left( \frac{P^\star(s| s,a)}{P( s'| s, a)} \right) + 2\log(\nicefrac{|\Mcal|}{\delta}).
\end{align}
By the same steps of obtaining to \Eqref{eq:Pmle_final}, we also have
\begin{align}
\label{eq:Rmle_final}
n \cdot \E_\mu \left[ \TV\left(R(\cdot|s,a), R^\star(\cdot|s,a)\right)^2 \right] \leq &~ \frac{1}{2} \sum_{(s,a,r') \in \Dcal}  \log \left( \frac{R^\star(s| s,a)}{R( s'| s, a)} \right) + 2\log(\nicefrac{|\Mcal|}{\delta}).
\end{align}
Combining \Eqref{eq:Pmle_final,eq:Rmle_final}, we obtain
\begin{align*}
&~ n \cdot \E_\mu \left[ \TV\left(P(\cdot|s,a), P^\star(\cdot|s,a)\right)^2 + \TV\left(R(\cdot|s,a), R^\star(\cdot|s,a)\right)^2 \right]
\\
\lesssim &~ \sum_{(s,a,s') \in \Dcal}  \log \left( \frac{P^\star(s| s,a)}{P( s'| s, a)} \right) + \sum_{(s,a,r') \in \Dcal}  \log \left( \frac{R^\star(s| s,a)}{R( s'| s, a)} \right) + \log(\nicefrac{|\Mcal|}{\delta})
\\
= &~ \log\ell_{\Dcal}(M^\star) - \log\ell_\Dcal(M) + \log(\nicefrac{|\Mcal|}{\delta}).
\tag{$\ell_{\Dcal}(\cdot)$ is defined in \Eqref{eq:def_LD}}
% \\
% \leq &~ \max_{M' \in \Mcal} \log\ell_{\Dcal}(M') - \log\ell_\Dcal(M) + \log(\nicefrac{|\Mcal|}{\delta}).
% \tag{by \cref{asm:realizability}}
\end{align*}
This completes the proof.
\end{proof}

\subsection{Proof of Main Theorems}

\begin{proof}[\cpfname{thm:perf}]
By the optimality of $\pihat$ (from \cref{eq:def_pihat}), we have
\begin{align}
J(\picom) - J(\pihat) = &~ J(\picom) - J(\piref) - \left[ J(\pihat) - J(\piref) \right]
\nonumber
\\
\leq &~ J(\picom) - J(\piref) - \min_{M \in \Mcal_\alpha} \left[ J_M(\pihat) - J_M(\piref) \right]
\tag{by \cref{lem:MLE_Mstar}, we have $M^\star \in \Mcal_\alpha$}
\\
\leq &~ J(\picom) - J(\piref) - \min_{M \in \Mcal_\alpha} \left[ J_M(\picom) - J_M(\piref) \right],
\label{eq:perfbdeq1}
\end{align}
where the last step is because of $\picom \in \Pi$
By the simulation lemma (\cref{lem:simulation}), we know for any policy $\pi$ and any $M \in \Mcal_\alpha$, 
\begin{align}
\left| J(\pi) - J_M(\pi) \right| \leq &~ \frac{\Vmax}{1 - \gamma} \E_{d^\pi} \left[ \TV\left(P_M(\cdot|s,a), P^\star(\cdot|s,a)\right) \right] + \frac{1}{1 - \gamma} \E_{d^\pi} \left[\left|R_M(s,a) - R^\star(s,a)\right| \right]
\nonumber
\\
\leq &~ \frac{\Vmax}{1 - \gamma} \sqrt{\E_{d^\pi} \left[ \TV\left(P_M(\cdot|s,a), P^\star(\cdot|s,a)\right)^2 \right]} + \frac{1}{1 - \gamma} \sqrt{ \E_{d^\pi} \left[\left(R_M(s,a) - R^\star(s,a)\right)^2 \right]}
\nonumber
\\
\lesssim &~ \frac{\Vmax}{1 - \gamma} \sqrt{\E_{d^\pi} \left[ \TV\left(P_M(\cdot|s,a), P^\star(\cdot|s,a)\right)^2 + \left(R_M(s,a) - R^\star(s,a)\right)^2 \right]}
\tag{$a \lesssim b$ means $a \leq \Ocal(b)$}
\\
\leq &~ \frac{\Vmax\sqrt{\C_\Mcal(\pi)}}{1 - \gamma} \sqrt{\E_{\mu} \left[ \TV\left(P_M(\cdot|s,a), P^\star(\cdot|s,a)\right)^2 + \left(R_M(s,a) - R^\star(s,a)\right)^2 \right]}
\nonumber
\\
\lesssim &~ \frac{\Vmax\sqrt{\C_\Mcal(\pi)}}{1 - \gamma} \sqrt{\frac{\max_{M' \in \Mcal} \Lcal_\Dcal(M') - \Lcal_\Dcal(M) +  \log(\nicefrac{|\Mcal|}{\delta})}{n}}
\tag{by \cref{lem:Vspace_M}}
\\
\lesssim &~ \frac{\Vmax\sqrt{\C_\Mcal(\pi)}}{1 - \gamma} \sqrt{\frac{\log(\nicefrac{|\Mcal|},{\delta})}{n}}
\label{eq:perfbdeq2}
\end{align}
where the last step is because $\max_{M' \in \Mcal} \Lcal_\Dcal(M') - \Lcal_\Dcal(M) \leq \alpha = \Ocal(\log(\nicefrac{|\Mcal|}{\delta}) / n$ by \cref{eq:v_space}.

Combining \cref{eq:perfbdeq1,eq:perfbdeq2}, we obtain
\begin{align*}
J(\picom) - J(\pihat) \lesssim &~ \left[\sqrt{\C_\Mcal(\picom)} + \sqrt{\C_\Mcal(\piref)}\right] \cdot \frac{\Vmax}{1 - \gamma} \sqrt{\frac{\log(\nicefrac{|\Mcal|},{\delta})}{n}}.
\end{align*}
This completes the proof.
\end{proof}

\begin{proof}[\cpfname{thm:RPI}]
\begin{align*}
J(\piref) - J(\pihat) = &~ J(\piref) - J(\piref) - \left[ J(\pihat) - J(\piref) \right]
\nonumber
\\
\leq &~ - \min_{M \in \Mcal_\alpha} \left[ J_M(\pihat) - J_M(\piref) \right]
\tag{by \cref{lem:MLE_Mstar}, we have $M^\star \in \Mcal_\alpha$}
\\
= &~ - \max_{\pi \in \Pi}\min_{M \in \Mcal_\alpha} \left[ J_M(\pi) - J_M(\piref) \right]
\tag{by the optimality of $\pihat$ from \cref{eq:def_pihat}}
\\
\leq &~ - \min_{M \in \Mcal_\alpha} \left[ J_M(\piref) - J_M(\piref) \right]
\tag{$\piref \in \Pi$}
\\
= &~ 0.
\end{align*}
\end{proof}
\section{Proofs for \cref{sec:RPI}}

% \begin{lemma}[Fixed-point Lemma] \label{lm:fixed-point lm of mb-atac}
% For any $\Mcal \subseteq \Mcal_\alpha$ and any $\psi:\Mcal\to\mathbb{R}$, consider the policy
% \begin{align} \label{eq:candidate fixed-point policies}
%     \pi \in \argmax_{\pi'\in\Pi} \min_{M\in\Mcal} J_M(\pi') + \psi(M)
% \end{align}
% Then $\pi$ is a fixed point in Definition~\ref{def:fixed_point}.
% Conversely, for any fixed point $\pi$ in \cref{def:fixed_point}, there is a $\psi:\Mcal\to\mathbb{R}$ such that $\pi$ is a solution to \eqref{eq:candidate fixed-point policies}.
% \end{lemma}
% \subsection{Proof of \cref{lm:fixed-point lm of mb-atac}}

\begin{proof}[\cpfname{lm:fixed-point lm of mb-atac}]
We prove the result by contradiction.
First notice $\min_{M\in\Mcal} J_M(\pi') - J_M(\pi') = 0$.
Suppose there is $\overline{\pi} \in \Pi$ such that $\min_{M\in\Mcal_\alpha} J_M(\bar{\pi}) - J_M(\pi') >0$, which implies that $J_M(\bar{\pi}) >J_M(\pi') $, $\forall M \in \Mcal_\alpha$.
%Let $\overline{M} = \argmin_{M\in\Mcal} J_M(\bar{\pi}) - J_M(\pi')$. This would imply that 
Since $\Mcal \subseteq \Mcal_\alpha$, we have 
\begin{align*}
    \min_{M\in\Mcal} J_M(\bar{\pi})  + \psi(M) 
    &>  \min_{M\in\Mcal}  J_M(\pi')  + \psi(M) =  \max_{\pi\in\Pi}\min_{M\in\Mcal}  J_M(\pi)  + \psi(M)
\end{align*}
which is a contradiction of the maximin optimality. Thus $\max_{\pi\in\Pi} \min_{M\in\Mcal_\alpha} J_M(\bar{\pi}) - J_M(\pi') = 0 $, which means $\pi'$ is a solution.

For the converse statement, suppose $\pi$ is a fixed point. We can just let $\psi(M) = -J_M(\pi)$. Then this pair of $\pi$ and $\psi$ by definition of the fixed point satisfies \eqref{eq:candidate fixed-point policies}.
\end{proof}

%%%%%%%%%%%%%%%%%%%%%%%%%%%%%%%%%%%%%%%%%%%%%%%%%%%%%%%%%%

\end{document}